\documentclass{article} 
\usepackage{collas2022_conference,times}

\usepackage{amsmath,amsfonts,bm}

\def\eqref#1{equation~\ref{#1}}

\def\1{\bm{1}}

\DeclareMathAlphabet{\mathsfit}{\encodingdefault}{\sfdefault}{m}{sl}
\SetMathAlphabet{\mathsfit}{bold}{\encodingdefault}{\sfdefault}{bx}{n}

\DeclareMathOperator*{\argmin}{arg\,min}

\usepackage{hyperref}
\usepackage{hyperref}
\usepackage{graphicx}
\usepackage{amsmath}
\usepackage{verbatim}
\usepackage{tikz}
\usetikzlibrary{shapes,trees,positioning,arrows}
\usepackage{url}
\usepackage{pdflscape}
\usepackage[linesnumbered, ruled]{algorithm2e}
\usepackage{multirow}
\usepackage{fixltx2e}
\usepackage{afterpage}
\usepackage{adjustbox}
\usepackage{multirow}
\usepackage{graphicx}
\usepackage{amsmath}
\usepackage{amssymb}
\usepackage{amsfonts}
\usepackage{multirow}
\usepackage{algorithm2e}
\usepackage{float}
\usepackage{enumitem}
\usepackage{icomma}
\usepackage{tikz}
\usetikzlibrary{shapes.geometric, arrows,chains}
\usepackage{array}
\usepackage{color, colortbl}
\usepackage{array}
\usetikzlibrary{shapes.geometric, arrows,chains}
\usepackage{multicol}
\usepackage{etoolbox}
\usetikzlibrary{fit,positioning}
\pgfdeclarelayer{l1}
\pgfdeclarelayer{l2}
\pgfsetlayers{l1,l2,main}
\usepackage[font={small}]{caption}
\usepackage[figuresright]{rotating}
\usepackage{xr}
\usepackage{mathrsfs}
\usepackage{rotating}
\usepackage{lscape}
\usepackage{amsthm}
\DeclareFontFamily{OT1}{pzc}{}
\DeclareFontShape{OT1}{pzc}{m}{it}{<-> s * [1.2] pzcmi7t}{}
\DeclareMathAlphabet{\mathpzc}{OT1}{pzc}{m}{it}
\newtheorem{theorem}{Theorem}

\newtheorem{corollary}{Corollary}
\makeatletter
\def\mainmatter{%
  \cleardoublepage
  \@mainmattertrue
  \pagenumbering{arabic}
  \def\mainmatter{\cleardoublepage\@mainmattertrue}
}
\pgfkeys{%
  /tikz/node on layer/.code={
    \gdef\node@@on@layer{%
      \setbox\tikz@tempbox=\hbox\bgroup\pgfonlayer{#1}\unhbox\tikz@tempbox\endpgfonlayer\egroup}
    \aftergroup\node@on@layer
  },
  /tikz/end node on layer/.code={
    \endpgfonlayer\endgroup\endgroup
  }
}
\def\node@on@layer{\aftergroup\node@@on@layer}
\makeatother
\usepackage{hyperref}
\newtheorem{definition}{Definition}

\hypersetup{
    colorlinks=true,
    linkcolor=red,
    filecolor=magenta,      
    urlcolor=blue,
    citecolor=purple,
    pdftitle={Overleaf Example},
    pdfpagemode=FullScreen,
    }

\title{A Theory for Knowledge Transfer in Continual Learning}

\author{Diana Benavides-Prado  \\
School of Computer Science\\
The University of Auckland\\
New Zealand \\
\texttt{d.benavides-prado@auckland.ac.nz} \\
\And 
Patricia Riddle  \\
School of Computer Science\\
The University of Auckland\\
New Zealand \\
\texttt{p.riddle@auckland.ac.nz} \\
}

\collasfinalcopy 

\begin{document}

\maketitle

\begin{abstract}
Continual learning of a stream of tasks is an active area in deep neural networks. The main challenge investigated has been the phenomenon of catastrophic forgetting or interference of newly acquired knowledge with knowledge from previous tasks. Recent work has investigated forward knowledge transfer to new tasks. Backward transfer for improving knowledge gained during previous tasks has received much less attention. There is in general limited understanding of how knowledge transfer could aid tasks learned continually. We present a theory for knowledge transfer in continual supervised learning, which considers both forward and backward transfer. We aim at understanding their impact for increasingly knowledgeable learners. We derive error bounds for each of these transfer mechanisms. These bounds are agnostic to specific implementations (\textit{e.g.} deep neural networks). We demonstrate that, for a continual learner that observes related tasks, both forward and backward transfer can contribute to an increasing performance as more tasks are observed. 
\end{abstract}

\section{Introduction}
Learning a continual stream of tasks has been a long-standing challenge in machine learning \citep{ring1997child,chen2018lifelong}. Continual learning with deep neural networks has been an active area of research over the past few years \citep{delange2021continual}, and it has multiple applications in a range of problem domains \citep{lesort2020continual,lee2020clinical,maschler2021regularization}. Catastrophic forgetting of existing knowledge for tasks learned sequentially has been the main challenge \citep{delange2021continual}. A variety of methods for this problem in supervised continual learning have been proposed, including approaches for replaying examples \citep{lopez2017gradient}, regularisation-based methods \citep{kirkpatrick2017overcoming} and network expansion methods \citep{ostapenko2019learning}. 

Knowledge transfer has recently been explored as an alternative for improving the performance of continual learning systems. Transferring knowledge in the forward direction has demonstrated some gains \citep{ke2021achieving}. Backward transfer on the other hand has been paid much less attention in continual learning with deep neural networks \citep{riemer2018learning,ke2020continual,vogelstein2020representation,new2022lifelong}. However, backward transfer has succeeded in other lifelong learning studies that use techniques such as Support Vector Machines (SVMs) \citep{benavides2020towards}, and continues to be a desired property of continual learning systems \citep{irina2022}.

We develop a theory for knowledge transfer in continual learning. We first derive error bounds for individual tasks, when these are subject to forward transfer when learned for the first time, or to backward transfer from future tasks when these are learned. We then consider the order of arrival of tasks, since this influences the the amount of transfer that task is subject to. Based on the bounds derived for individual tasks, we calculate error bounds for a continual learner that learns related tasks sequentially using forward and backward transfer.

Our framework relies on three core assumptions. First, the continual learner is embedded into an \textit{environment} of related tasks. This allows us to treat the problem of learning a sequence of tasks as the problem of learning a bias for the whole environment incrementally. Learning this bias is helpful since the continual learning will perform better at any task in that environment. Our second assumption is that relatedness between these tasks relies on the \textit{similarity} between their example generating distributions. This assumption allows us to use a set of \textit{transformation functions} as a tool for constraining the hypothesis family for learning a particular task, based on its similarity to other tasks in the environment (from which forward or backward transfer are to be performed). This tool has been used in other studies in multitask learning \citep{ben2008notion}. Our final assumption is that each task has a sufficient number of examples from which to learn. This assumption distinguishes our framework from approaches in zero-shot or few-shot learning. However, in Section \ref{sec:forward} we demonstrate that the number of examples required to learn decreases with the number of tasks.

Our proposed framework is generic and does not rely on any practical assumptions about the implementation of the continual learner (\textit{e.g.} in terms of the learning technique, the implementation architecture, the mechanisms used for transfer or the continual learning scenario - task-incremental, domain-incremental or class-incremental \citep{van2019three}). We aim to provide a rigorous theoretical analysis to show the potential of knowledge transfer while learning sequentially, and to encourage more research in this direction. 

This paper is organised as follows. Section \ref{sec:previous_research} describes previous research in knowledge transfer for continual learning. Section \ref{sec:preliminaries} provides preliminaries and notation. Section \ref{sec:forward} describes error bounds derived for tasks learned continually using forward knowledge transfer. Section \ref{sec:backward} describes error bounds for tasks learned continually using backward transfer. Section \ref{sec:continual} describes error bounds for a continual learner that uses both forward and backward transfer. Finally, Section \ref{sec:discussion_conclusion} provides some discussion and final remarks.

\section{Previous Research}
\label{sec:previous_research}
Catastrophic forgetting or interference of new tasks with previously acquired knowledge has been studied extensively in supervised continual learning with deep neural networks \citep{delange2021continual}. Several methods to avoid catastrophic forgetting have been proposed, ranging from example replay \citep{lopez2017gradient,van2020brain} to regularisation-based \citep{kirkpatrick2017overcoming,zeng2019continual} to dynamic networks \citep{yoon2017lifelong,hung2019compacting}.  Beyond catastrophic forgetting, the classic aim of continual learning systems has been to achieve increasingly knowledgeable systems \citep{ring1997child,chen2018lifelong}. Knowledge transfer has been proposed as a mechanism to achieve this \citep{ke2020continual,rostami2020using,benavides2020beyond}. Forward transfer with continual deep neural networks has been studied recently \citep{ke2021achieving}. Backward transfer, in contrast, has received much less attention \citep{riemer2018learning,ke2020continual,vogelstein2020representation}, although it was explored with alternative techniques such as SVM \citep{benavides2020towards}.

Ben-David and Borbely (2008)\nocite{ben2008notion} and Baxter (2000)\nocite{baxter2000model} studied the effects of learning multiple related tasks jointly with multitask learning. Baxter (2000) derived the expected average error for a group of tasks learned jointly. Ben-David and Borbely (2008) derived similar bounds for a single task learned under the same framework. More recently, Benavides-Prado, Koh and Riddle (2020) \nocite{benavides2020towards} derived error bounds of knowledge transfer across SVM models in supervised continual learning. This research showed that given a set of related tasks, backward transfer with SVM can be used to achieve systems that improve their performance with each incoming task \citep{benavides2020towards}. Furthermore, forward transfer can also be used to aid learning of new tasks. Although novel, these bounds were specific to the implementation using SVM. Here we extend this work by deriving error bounds that are agnostic to the implementation, for both forward transfer and backward transfer. We also derive error bounds for a continual learner that uses knowledge transfer whilst learning related tasks sequentially. 

Other theoretical frameworks in transfer learning have studied how the degree of relatedness among tasks helps transfer  \citep{lampinen2018analytic}, and how transfer helps curriculum learning \citep{weinshall2018curriculum}. Theoretical studies in continual learning have studied the effects of task similarity in catastrophic forgetting \citep{lee2021continual}, and discovered that optimal continual learning is NP-hard and requires perfect memory \citep{knoblauch2020optimal}. However, to the best of our knowledge there is no prior study that evaluates the effects of forward and backward knowledge transfer in learning a continual stream of supervised tasks.

\section{Preliminaries and Definitions}
\label{sec:preliminaries}
Supervised continual learning is about learning a stream of tasks $\mathbf{T} = \{T_{1}, \ldots, T_{n}\}$. A given task $T$ in the sequence has an underlying probability distribution $P(X, Y)$ (or simply $P$, which we use later indistinctly). For that task, the aim is to learn a function $f: X \rightarrow Y$, that maps the input space $X$ to the output space $Y$. Learning works by exploring a hypothesis space $H$ on that task, and finding the hypothesis $h \in H$ such that: 

\begin{equation}
\begin{gathered}
Er^{P}(h) = min_{h \in H} \mathcal{L}(h(x), y)
\end{gathered}
\end{equation}

where $\mathcal{L}$ is a loss function. Naturally, estimating the error of $h$ over the actual distribution $P$ is difficult since $P$ can not be observed directly. Instead, a sample $S$ of $m$ examples extracted repeatedly from $P$ is used such that:

\begin{equation}
\begin{gathered}
S = \{(x_{1}, y_{1}), (x_{2}, y_{2}), \ldots, (x_{m}, y_{m})\}
\end{gathered}
\end{equation}

And the empirical error of $h \in H$ over $S$ is such that:
\begin{equation}
\begin{gathered}
\hat{Er}^{S}(h) = min_{h \in H} \dfrac{1}{m} \sum_{i=1}^{m} \mathcal{L}(h(x_{i}), y_{i})
\end{gathered}
\label{eq:empirical_S}
\end{equation}

To find the best $h$ that satisfies Eq. \ref{eq:empirical_S}, the learner aims to find the hypothesis that best fits this sample better, such that:

\begin{equation}
\begin{gathered}
h^{*} = \inf_{h \in H} \mathcal{L}(h(x), y)
\end{gathered}
\end{equation}

Knowledge transfer for continual learning aims to share knowledge across tasks observed sequentially. In our framework, we distinguish two types of transfer: 1) forward transfer, which aims to learn new or \textit{target} tasks better or faster by transferring knowledge gained during tasks learned earlier, and 2) backward transfer, which aims to improve future performance over previous or \textit{source} tasks by using knowledge collected while learning new tasks. We assume that tasks observed by the continual learner are related. Therefore, these tasks are assumed to belong to the same \textit{environment}, and the continual learner can become better at learning in this environment as more tasks are observed.

Formally, we define the environment of the continual learner as follows:

\begin{definition}
An environment $(\mathcal{P}, \mathcal{Q})$ of related tasks, corresponds to the set of all probability distributions on $\mathcal{X} \times \mathcal{Y}$,  denoted $\mathcal{P}$ , and a distribution on $\mathcal{P}$, denoted $\mathcal{Q}$. Instead of exploring a single hypothesis space, the continual learner has access to a family of hypothesis spaces $\mathbb{H} = \{H_{1}, H_{2}, \ldots, H_{n}\}$, one for each task. In practice, the learner has access to multiple samples to learn from, one sample for each task, such that $\mathbf{S} = \{S_{1}, \ldots, S_{n}\}$ are drawn at random from underlying probability distributions $\mathbf{P} = \{P_{1}, \ldots, P_{n} \}$. 
\label{def:environment}
\end{definition}

Access to a family of hypothesis spaces rather than a single hypothesis space, as in single-task learning, gives the continual learner the potential to learn a good bias that can generalise well to novel tasks from the same environment. Rather than producing a hypothesis that with high probability will perform well on future examples of a particular task, by learning related tasks continually the learner will produce a hypothesis space that with high probability will perform well on future tasks within the same environment. This main result has been demonstrated in the context of multitask learning \citep{baxter2000model}, and is the main result we demonstrate in Sections \ref{sec:forward}-\ref{sec:continual} for a stream of tasks learned continually using knowledge transfer.

The notion of relatedness for tasks in the environment of the continual learner relies on the similarity of their example generating distributions \citep{ben2008notion}. Formally, given a set of transformation functions $\mathpzc{f} \in \mathcal{F}$ such that $\mathpzc{f} : \mathcal{X} \rightarrow \mathcal{X}$, tasks in the environment are $\mathcal{F}$-related if, for some fixed probability distribution over $\mathcal{X} \times Y$, if the examples in each of these tasks can be generated by applying some $\mathpzc{f} \in \mathcal{F}$ to that distribution. Therefore, we can define the equivalence relation \citep{raczkowski1990equivalence} $\sim_{\mathcal{F}}$ on $\mathbb{H}$, where $\mathbb{H}$ is a family of hypothesis spaces for all tasks in the environment, as follows: 

\begin{definition} Let $\{P_{1}, \ldots, P_{n}\}$ be the underlying probability distributions of a set of $n$ tasks over a domain $\mathcal{X} \times Y$. Let $\mathcal{F}$ be a set of transformations $\mathpzc{f}: \mathcal{X} \rightarrow \mathcal{X}$. Let $P_{1}$ and $P_{2}$ be related if one can be generated from the other by applying some $\mathpzc{f} \in \mathcal{F}$, such that $P_{1} = \mathpzc{f}[P_{2}]$ (and therefore $P_{2} = \mathpzc{f}^{-1}[P_{1}]$)  or $P_{2} = \mathpzc{f}[P_{1}]$ (and therefore $P_{1} = \mathpzc{f}^{-1}[P_{2}]$). The samples $\{S_{1}, \ldots, S_{n}\}$ to be used during learning tasks $\{T_{1}, \ldots, T_{n}\}$ are said to be $\mathcal{F}$-related if these samples come from $\mathcal{F}$-related probability distributions. 

Let $\mathbb{H}$ be a family of hypothesis spaces over the domain $\mathcal{X} \times Y$, and $\mathbb{H}$ be closed under the action of $\mathcal{F}$. Let $\mathcal{H}$ be a family of hypothesis spaces that consist of sets of hypotheses $[\mathpzc{h}] \in \mathbb{H}$ which are equivalent up to transformations in $\mathcal{F}$. If $\mathcal{F}$ acts as a group over $\mathbb{H}$ because:
\begin{itemize}
\item For every $\mathpzc{f} \in \mathcal{F}$ and every $\mathpzc{h} \in \mathbb{H}$, $\mathpzc{h} \circ \mathpzc{f} \in \mathbb{H}$, and
\item $\mathcal{F}$ is closed under transformation composition and inverses, \textit{i.e.} for every $\mathpzc{f}, \mathpzc{g} \in \mathcal{F}$, the inverse transformation, $\mathpzc{f}^{-1}$, and the composition, $\mathpzc{f} \circ \mathpzc{g}$ are also members of $\mathpzc{F}$
\end{itemize}
Then the equivalence relation $\sim_{\mathcal{F}}$ on $\mathbb{H}$ is defined by: 
$\mathpzc{h_{1}}\sim_{\mathcal{F}}\mathpzc{h_{2}} \iff$ there exists $\mathpzc{f} \in \mathcal{F}$ such that $\mathpzc{h_{2}} = \mathpzc{h_{1}} \circ \mathpzc{f}$.

Therefore this framework considers the family of hypothesis spaces $\mathcal{H} = \{[\mathpzc{h}] : [\mathpzc{h}] \in \mathbb{H}\}$, which is the family of all equivalence classes of $\mathbb{H}$ under $\sim_{\mathcal{F}}$. 
\label{def:initial}
\end{definition}

The original setting of this framework is in multitask learning \citep{ben2008notion}, where the equivalence class $[\mathpzc{h}]$ for a target task is first found using samples from all tasks. This requires to first identify aspects of all tasks that are invariant under $\mathcal{F}$. A second step restricts the learning of a particular task to selecting a hypothesis $\mathpzc{h'} \in [\mathpzc{h}]$ as the hypothesis for that task. Therefore, the target task can benefit from transfer during this second step by exploring the  hypothesis space to be explored for the target task $[\mathpzc{h}]$ that contains these invariances.

In continual learning we are faced with a similar problem, but rather than learning tasks jointly these are observed sequentially. However, provided these tasks are $\mathcal{F}$-related, we can adopt a similar framework to derive error bounds of a target task that is learned with forward transfer from a set of source tasks, and of source tasks for which knowledge is updated with backward transfer from a recently learned target tas. In the following sections we develop a theory of knowledge transfer across continual tasks that use these two transfer mechanisms. 

\section{Forward Knowledge Transfer across Related Tasks}
\label{sec:forward}
In this and following sections, we will use $^{(t)}$ to refer to a target task, or target probability distribution or target sample, and ${^{(s)}}$ to denote a source task, or source probability distribution or source sample. In forward transfer, the aim is to learn a \textit{target} task $T^{(t)}$ helped by knowledge obtained during previous $n$ source tasks $\{T^{(s)}_{1}, \ldots, T^{(s)}_{n}\}$, with probability distributions $P^{(t)}$ and $\{P^{(s)}_{1}, \ldots, P^{(s)}_{n}\}$ and their corresponding observed samples $S^{(t)}$ and $\{S^{(s)}_{1}, \ldots, S^{(s)}_{n}\}$. Forward transfer for a continual learner which observes $\mathcal{F}$-related tasks is defined as follows: 
\begin{definition}
Given classes $\mathcal{F}$ and $\mathcal{H}$, and a set of labeled samples $\{S^{(s)}_{1}, \ldots, S^{(s)}_{n}\}$ for a set of $n$ source tasks and a labeled sample $S^{(t)}$ for a target task, in forward knowledge transfer while learning task $T^{(t)}$, the continual learner: 
\begin{enumerate}
\item Has access to $[\mathpzc{h}^*] \in \mathcal{H}$, obtained as a result of minimising $\inf_{\mathpzc{h}_{1}, \ldots, \mathpzc{h_{n}} \in [\mathpzc{h}]} \sum_{i=1}^{n} \mathbb{\hat{E}}r^{S^{(s)}_{i}}(\mathpzc{h}_{i})$ over all $\mathpzc{[h]} \in \mathcal{H}$.
\item Selects $\mathpzc{h^{\Diamond}} \in \mathpzc{[h^*]} $ that minimises $\mathbb{\hat{E}}r^{S^{(t)}}(\mathpzc{h'})$ over all $\mathpzc{h'} \in \mathpzc{[h^*]}$, and outputs $\mathpzc{h^{\Diamond}}$ as the hypothesis for $T^{(t)}$.
\end{enumerate}
\label{def:definition_forward}
\end{definition}

In practice, having access to $[\mathpzc{h^{*}}]$ during a target task $T^{(t)}$ implies that the continual learner can access to some representation of the knowledge obtained during previous tasks (\textit{e.g.} access to a neural network representing that knowledge). We derive error bounds for learning a target task $T^{(t)}$ helped by knowledge transfer from $\mathcal{F}$-related source tasks as follows:

\begin{theorem} Let $\{P^{(s)}_{1}$, \ldots, $P^{(s)}_{n}\}$ and $P^{(t)}$ be a set of $\mathcal{F}$-related probability distributions, and $\{S^{(s)}_{1}$, \ldots, $S^{(s)}_{n}\}$ and $S^{(t)}$ random samples representing these distributions. Let $\mathcal{F}$ and $\mathcal{H}$ be defined as in Definition \ref{def:initial}. Let $d_{max} = sup\{VC\mbox{-}dim(H): H \in \mathcal{H}\}$. Let $d_{\mathcal{H}}(n) = max_{[\mathpzc{h}] \in \mathcal{H}}VC\mbox{-}dim([\mathpzc{h}])$. Let $\mathpzc{h^{\Diamond}}$ be selected according to Definition \ref{def:definition_forward}. Then, for every constant $\epsilon_{1}$, $\epsilon_{2}$, $\delta > 0$, with $|S^{(t)}|$ and $|S^{(s)}_{i}|$ defined similarly to Theorem $3$ in Ben-David and Borbely (2008): 
\begin{equation}
\begin{gathered}
|S^{(t)}| \geq \dfrac{64}{\epsilon_{1}^2} \bigg[2d_{max} log\dfrac{12}{\epsilon_{1}} + log\dfrac{8}{\delta}\bigg]
\end{gathered}
\label{eq:examples_target}
\end{equation}
and, for all i $\leq$ n:
\begin{equation}
\begin{gathered}
|S^{(s)}_{i}| \geq \dfrac{88}{\epsilon_{2}^{2}} \bigg[2d_{\mathcal{H}}(n) log{\dfrac{22}{\epsilon_{2}}} + \dfrac{1}{2} log \dfrac{8}
{\delta}\bigg]
\end{gathered}
\label{eq:examples_sources}
\end{equation}
then with probability greater than $(1 - \delta)$: 
\begin{equation}
\begin{gathered}
\mathbb{E}r^{P^{(t)}}(\mathpzc{h^\Diamond}) \leq \inf_{\mathpzc{h}\in \mathbb{H}} \mathbb{E}r^{P^{(t)}}(\mathpzc{h}) + 2(\epsilon_{1} + \epsilon_{2})
\end{gathered}
\end{equation}
\label{theom:forward}
\end{theorem}
\begin{proof} Let $\mathpzc{h^\#}$ be the best $P^{(t)}$ label predictor in $\mathbb{H}$, \textit{i.e.} $\mathpzc{h}^\# = \argmin_{\mathpzc{h} \in \mathbb{H}} \mathbb{E}r^{P^{(t)}}(\mathpzc{h})$. Let $[\mathpzc{h}^*]$ be the equivalence class picked according to Definition \ref{def:definition_forward}. By the choice of $\mathpzc{h}^*$:
\begin{equation}
\begin{gathered}
\inf_{\mathpzc{h_{1}}, \ldots, \mathpzc{h_{n}} \in \mathpzc{[h^*]}} \dfrac{1}{n} \sum_{i=1}^{n}\mathbb{\hat{E}}r^{S^{(s)}_{i}} (\mathpzc{h}_{i})  \leq \inf_{\mathpzc{h_{1}}, \ldots, \mathpzc{h_{n}} \in [\mathpzc{h^\#]}} \dfrac{1}{n} \sum_{i=1}^{n}\mathbb{\hat{E}}r^{S^{(s)}_{i}} (\mathpzc{h}_{i})
\end{gathered}
\end{equation}
By Theorem 2 in Ben-David and Borbely (2008), with probability greater than $(1 - \delta/2)$: 
\begin{equation}
\begin{gathered}
\inf_{\mathpzc{h_{1}}, \ldots, \mathpzc{h_{n}} \in \mathpzc{[h^\#]}} \dfrac{1}{n} \sum_{i=1}^{n}\mathbb{\hat{E}}r^{S^{(s)}_{i}} (\mathpzc{h}_{i}) \leq \mathbb{E}r^{P^{(t)}} ([\mathpzc{h}^{\#}]) + \epsilon_{1}
\end{gathered}
\end{equation}
and:
\begin{equation}
\begin{gathered}
\mathbb{E}r^{P^{(t)}} ([\mathpzc{h^*}]) \leq \inf_{\mathpzc{h_{1}, \ldots, \mathpzc{h_{n}} \in \mathpzc{[h^*]}}} \dfrac{1}{n} \sum_{i=1}^{n}\mathbb{\hat{E}}r^{S^{(s)}_{i}} (\mathpzc{h}_{i}) + \epsilon_{1}
\end{gathered}
\end{equation}
Then, combining the inequalities above, with probability greater than $(1 - \delta/2)$: 
\begin{equation}
\begin{gathered}
\mathbb{E}r^{P^{(t)}} ([\mathpzc{h^*}]) \leq \mathbb{E}r^{P^{(t)}} ([\mathpzc{h}^{\#}]) + 2 \epsilon_{1}
\end{gathered}
\end{equation}

Since $\mathpzc{h^\Diamond} \in [\mathpzc{h^*}]$, with probability greater than $(1 - \delta/2)$, $\mathpzc{h^\Diamond}$ will have an error for $P^{(t)}$ which is within $2\epsilon_{2}$ of the best hypothesis there, \textit{i.e.} $\mathbb{E}r^{P^{(t)}} ([\mathpzc{h^*}])$. Therefore:
\begin{equation}
\begin{gathered}
\mathbb{E}r^{P^{(t)}} (\mathpzc{h^\diamond}) \leq \mathbb{E}r^{P^{(t)}} (\mathpzc{h}^{\#}) + 2( \epsilon_{1} + \epsilon_{2})
\end{gathered}
\end{equation}
\end{proof}
Theorem \ref{theom:forward} implies that, for a sufficiently large number of examples for the sources and the target tasks, forward transfer is expected to benefit learning of a target task. This result is achieved by choosing a hypothesis space for $T^{(t)}$ which is biased towards the hypothesis space learned for previous $\mathcal{F}$-related tasks from the same environment. The extent of this benefit depends on the number of examples per task (see Eq. \ref{eq:examples_target} and Eq. \ref{eq:examples_sources}). Baxter (2000) demonstrated that the number of examples required per task decreases along with an increasing number of tasks, in particular:
\begin{equation}
\begin{gathered}
|S| = \mathcal{O}\Big(\dfrac{1}{n} log \mathcal{C}(\epsilon, \mathcal{H}_{l}^{n})\Big)
\end{gathered}
\label{eq:number_examples_tasks}
\end{equation}

where $\mathcal{C}(\epsilon, \mathcal{H}_{l}^{n})$ is the capacity of the learner given an error $\epsilon$ and a set of $n$ sets of loss functions $\{\mathbf{h}_{l}^{1}, \ldots, \mathbf{h}_{l}^{n} \} \in \mathcal{H}_{l}^{n}$ for the family of hypothesis spaces $\mathcal{H}$. Provided that this capacity increases sublinearly with $n$, the number of examples required per task will decrease with an increasing number of tasks.

The amount of transfer to a target task and therefore the extent to which the bound in Theorem \ref{theom:forward} is satisfied depends on how many source tasks are used for transfer. Intuitively, the larger this number, the smaller the bound, since the target task will have a better bias of its environment with more related tasks having been observed, which would lead to a better hypothesis space to be selected for that task. Therefore, the later a target task is observed, the greater the opportunity for it to benefit from forward transfer. This is in accordance with previous research that demonstrated that a larger number of tasks learned continually benefits transfer  \citep{benavides2017accgensvm,benavides2020towards}. Next we analyse the effect of the task order in forward transfer, and the error bounds of a target task depending on that order. Next we derive bounds for forward transfer that account for the order of the task being observed in the sequence.

\begin{definition}
Given classes $\mathcal{F}$ and $\mathcal{H}$, a set of labeled samples \{$S^{(s)}_{1}$, \ldots,  $S^{(s)}_{n}\}$ for a set of source tasks and a labeled sample $S^{(t)}$ for a target task. Let:
\begin{itemize}
\item $[\mathpzc{h}_{n}^*] \in \mathcal{H}$ be the result of minimising $\inf_{\mathpzc{h}_{1}, \ldots, \mathpzc{h_{n}} \in [\mathpzc{h_{n}^*}]} \dfrac{1}{n} \sum_{i=1}^{n} \mathbb{\hat{E}}r^{S^{(s)}_{i}} (\mathpzc{h_{i}})$ over all $\mathpzc{[h]} \in \mathcal{H}$, at time $n$.
\item $[\mathpzc{h}_{n+z}^*] \in \mathcal{H}$ be the result of minimising $\inf_{\mathpzc{h}_{1}, \ldots, \mathpzc{h_{n+z}} \in [\mathpzc{h_{n+z}^*}]} \dfrac{1}{n+z} \sum_{i=1}^{n+z} \mathbb{\hat{E}}r^{S^{(s)}_{i}} (\mathpzc{h_{i}})$ over all $\mathpzc{[h]} \in \mathcal{H}$, at time $n+z$.
\item $\mathpzc{h_{n}^{\Diamond}} \in \mathpzc{[h_{n}^*]} $ that minimises $\mathbb{\hat{E}}r^{S^{(t)}} (\mathpzc{h'})$ over all $\mathpzc{h'} \in \mathpzc{[h_{n}^*]}$, and outputs $\mathpzc{h_{n}^{\Diamond}}$ as the hypothesis for task $T^{(t)}$ at time $n$.
\item $\mathpzc{h_{n+z}^{\Diamond}} \in \mathpzc{[h_{n+z}^*]} $ that minimises $\mathbb{\hat{E}}r^{S^{(t)}} (\mathpzc{h'})$ over all $\mathpzc{h'} \in \mathpzc{[h_{n+z}^*]}$, and outputs $\mathpzc{h_{n+z}^{\Diamond}}$ as the hypothesis for task $T^{(t)}$ at time $n+z$.
\end{itemize}
\label{def:def_corollaryforward}
\end{definition}

\begin{corollary}
Let $\{P^{(s)}_{1}$, \ldots, $P^{(s)}_{n}\}$, $\{S^{(s)}_{1}$, \ldots, $S^{(s)}_{n}\}$ and $S^{(t)}$, $\mathcal{F}$, $\mathcal{H}$, $d_{max}$, $d_{\mathcal{H}}(n)$ be defined as in Theorem \ref{theom:forward}, at time $n$. Similarly, let $\{P^{(s)}_{1}$, \ldots, $P^{(s)}_{n+z}\}$ and $P^{(t)}$ be a set of $\mathcal{F}$-related probability distributions, $\{S^{(s)}_{1}$, \ldots, $S^{(s)}_{n+z}\}$ and $S^{(t)}$ random samples representing these distributions, at time $n+z$. Let $\mathpzc{h^{\Diamond}}_{n}$ and $\mathpzc{h^{\Diamond}}_{n+z}$ be selected according to Definition \ref{def:def_corollaryforward}, at time $n$ and $n+z$, respectively. Then, for every $\epsilon_{1}$, $\epsilon_{2}$, $\delta > 0$, if: 
\begin{equation}
\begin{gathered}
|S^{(t)}| \geq \dfrac{64}{\epsilon_{1}^2} \bigg[2d_{max} log\dfrac{12}{\epsilon_{1}} + log\dfrac{8}{\delta}\bigg]
\end{gathered}
\end{equation}
and, at time $n$, for all i $\leq$ n:
\begin{equation}
\begin{gathered}
|S^{(s)}_{i}| \geq \dfrac{88}{\epsilon_{n}^{2}} \bigg[2d_{\mathcal{H}}(n) log{\dfrac{22}{\epsilon_{n}}} + \dfrac{1}{2} log \dfrac{8}
{\delta}\bigg]
\end{gathered}
\end{equation}
while, at time $n+z$, for all i $\leq$ (n+z):
\begin{equation}
\begin{gathered}
|S^{(s)}_{i}| \geq \dfrac{88}{\epsilon_{n+z}^{2}} \bigg[2d_{\mathcal{H}}(n+z) log{\dfrac{22}{\epsilon_{n+z}}} + \dfrac{1}{2} log \dfrac{8}
{\delta}\bigg]
\end{gathered}
\end{equation}
then with probability greater than $(1 - \delta)$: 
\begin{equation}
\begin{gathered}
\mathbb{E}r^{P^{(t)}}_{n+z}(\mathpzc{h_{n+z}^\Diamond}) \leq \mathbb{E}r^{P^{(t)}}_{n}(\mathpzc{h_{n}^\Diamond}) + \epsilon_{n} + \epsilon_{n+z} 
\end{gathered}
\end{equation}
\label{cor:forward_order}
\end{corollary}

See Appendix A for the proof of this corollary. The main part of the proof in Appendix A lies in Eq. \ref{eq:inf_nz_n}. Since the best hypothesis space for a larger number of tasks is better than the best hypothesis space for a smaller number of tasks in the same environment, \textit{i.e.} the bias over the environment gets refined over time, tasks observed later in the sequence will benefit more from transfer.

Bounds in Theorem \ref{theom:forward} and Corollary \ref{cor:forward_order} depend on the difference between $d_{max}$ and  $d_{\mathcal{H}}(n)$, and $VC\mbox{-}dim(\mathbb{H})$, with $d_{max} = sup\{VC\mbox{-}dim(H): H \in \mathcal{H}\}$ and $d_{\mathcal{H}}(n) = max_{[\mathpzc{h}] \in \mathcal{H}}VC\mbox{-}dim([\mathpzc{h}])$, and $d_{max} \leq d_{\mathcal{H}}(n) \leq VC\mbox{-}dim(\mathbb{H})$. Ben-David and Borbely (2008) showed that, for a sufficiently large number or tasks $n$, $d_{max} = max_{[\mathpzc{h}] \in \mathcal{H}}VC\mbox{-}dim([\mathpzc{h}]) = d_{\mathcal{H}}(n)$. We refer readers to Section 6 of Ben-David and Borbely (2008) for details.

\section{Backward Knowledge Transfer across Related Tasks}
\label{sec:backward}
Backward transfer works by updating a source task $T^{(s)}$ using knowledge gained during the most recent target task $T^{(t)}$. Transfer occurs from the space of a target probability distribution $P^{(t)}$, represented by a sample $S^{(t)}$, to the space of a probability distribution $P^{(s)}$ that uses a sample $S^{(s)}$ for learning that source task. In continual learning, the aim is to use $P^{(t)}$, and its corresponding sample $S^{(t)}$, to bias the update of a refined version of $P^{(s)}$ towards aspects that are invariant with $P^{(t)}$, provided these are related. Benavides-Prado, Koh and Riddle (2020), analysed the special case of two tasks, one source $T^{(s)}$ and one target $T^{(t)}$, for a specific implementation of a continual learner based on SVM. Here, we present bounds for an agnostic continual learner, as follows:
\begin{definition}
Given classes $\mathcal{F}$ and $\mathcal{H}$, and a pair of labeled samples $S^{(s)}$, $S^{(t)}$ for tasks $T^{(s)}$, $T^{(t)}$, during backward transfer the continual learner: 
\begin{enumerate}
\item Selects $[\mathpzc{h}^*] \in \mathcal{H}$ that minimises $\inf_{\mathpzc{h}^{(s)}, \mathpzc{h^{(t)}} \in [\mathpzc{h}]} \big(\mathbb{\hat{E}}r^{S^{(s)}} (\mathpzc{h}^{(s)}) + \mathbb{\hat{E}}r^{S^{(t)}} (\mathpzc{h}^{(t)})\big)$ over all $\mathpzc{[h]} \in \mathcal{H}$.
\item Selects $\mathpzc{h^{\Diamond}} \in \mathpzc{[h^*]} $ that minimises $\mathbb{\hat{E}}r^{S^{(s)}} (\mathpzc{h'})$ over all $\mathpzc{h'} \in \mathpzc{[h^*]}$, and outputs $\mathpzc{h^{\Diamond}}$ as the hypothesis for task $T^{(s)}$.
\end{enumerate}
\label{def:def_backward}
\end{definition}

In practice, the two steps in Definition \ref{def:def_backward} could be performed sequentially or jointly. For example, selecting $[\mathpzc{h}^*]$ in the first step could be performed by jointly training an auxiliary learner with examples from both $T^{(s)}$ and $T^{(t)}$, and then transferring back this information to $T^{(s)}$ during the second step. Alternatively, both $[\mathpzc{h^{*}}]$ could be selected jointly while training for $T^{(s)}$ aided by $T^{(t)}$.

Based on Definition \ref{def:def_backward}, in the special case of two tasks $T^{(s)}$ and $T^{(t)}$:
\begin{theorem} Let $P^{(s)}$ and $P^{(t)}$ be a set of $\mathcal{F}$-related probability distributions,and  $S^{(s)}$ and $S^{(t)}$ random samples representing these distributions on tasks $T^{(s)}$ and $T^{(t)}$ respectively. Let $\mathcal{F}$ and $\mathcal{H}$ be defined as in Definition \ref{def:initial}. Let $d_{max}$ and $d_{\mathcal{H}}(n)$ be defined as in Theorem \ref{theom:forward}. Let $\mathpzc{h^{\Diamond}}$ be selected according to Definition \ref{def:def_backward}. Then, for every $\epsilon_{1}$, $\epsilon_{2}$, $\delta > 0$, if: 
\begin{equation}
\begin{gathered}
|S^{(s)}| \geq \dfrac{64}{\epsilon_{1}^2} \bigg[2d_{max} log\dfrac{12}{\epsilon_{1}} + log\dfrac{8}{\delta}\bigg]
\end{gathered}
\end{equation}
and:
\begin{equation}
\begin{gathered}
|S^{(t)}| \geq \dfrac{88}{\epsilon_{2}^{2}} \bigg[2d_{\mathcal{H}}(2) log{\dfrac{22}{\epsilon_{2}}} + \dfrac{1}{2} log \dfrac{8}
{\delta}\bigg]
\end{gathered}
\end{equation}
then with probability greater than $(1 - \delta)$: 
\begin{equation}
\begin{gathered}
\mathbb{E}r^{P^{(s)}}(\mathpzc{h^\Diamond}) \leq \inf_{\mathpzc{h}\in \mathbb{H}} \mathbb{E}r^{P^{(s)}}(\mathpzc{h}) + 2(\epsilon_{1} + \epsilon_{2})
\end{gathered}
\end{equation}
\label{theom:backward}
\end{theorem}
\begin{proof} Let $\mathpzc{h^\#}$ be the best $P^{(s)}$ label predictor in $\mathbb{H}$, \textit{i.e.} $\mathpzc{h}^\# = \argmin_{\mathpzc{h} \in \mathbb{H}} \mathbb{E}r^{P^{(s)}}(h)$. Let $[\mathpzc{h}^*]$ be the equivalence class picked according to Definition \ref{def:def_backward}. By the choice of $\mathpzc{h}^*$:
\begin{equation}
\begin{gathered}
\inf_{\mathpzc{h^{(s)}}, \mathpzc{h^{(t)}} \in \mathpzc{[h^*]}} \big(\mathbb{\hat{E}}r^{S^{(s)}} (\mathpzc{h}^{(s)}) + \mathbb{\hat{E}}r^{S^{(t)}} (\mathpzc{h}^{(t)})\big) \leq \inf_{\mathpzc{h^{(s)}}, \mathpzc{h^{(t)}} \in [\mathpzc{h^\#]}}\big(\mathbb{\hat{E}}r^{S^{(s)}} (\mathpzc{h}^{(s)}) + \mathbb{\hat{E}}r^{S^{(t)}} (\mathpzc{h}^{(t)})\big)
\end{gathered}
\end{equation}
By Theorem 2 in Ben-David and Borbely (2008), in the case of two tasks: 
\begin{equation}
\begin{gathered}
\bigg|\mathbb{E}r^{P^{(s)}}([\mathpzc{h}]) - \inf_{h^{(s)}, h^{(t)} \in [\mathpzc{h}]} \dfrac{1}{2} (\hat{\mathbb{E}}r^{S^{(s)}} (\mathpzc{h}^{(s)}) + \hat{\mathbb{E}}r^{S^{(t)}} (\mathpzc{h}^{(t)}))\bigg| \leq \epsilon_{1}
\end{gathered}
\end{equation}
then with probability greater than $(1 - \delta/2)$: 
\begin{equation}
\begin{gathered}
\inf_{\mathpzc{h^{(s)}}, \mathpzc{h^{(t)}} \in \mathpzc{[h^\#]}} \big(\mathbb{\hat{E}}r^{S^{(s)}} (\mathpzc{h}^{(s)}) + \mathbb{\hat{E}}r^{S^{(t)}} (\mathpzc{h}^{(t)})\big) \leq \mathbb{E}r^{P^{(s)}} ([\mathpzc{h}^{\#}]) + \epsilon_{1}
\end{gathered}
\end{equation}
and:
\begin{equation}
\begin{gathered}
\mathbb{E}r^{P^{(s)}} ([\mathpzc{h^*}]) \leq \inf_{\mathpzc{h^{(s)}, \mathpzc{h^{(t)}} \in \mathpzc{[h^*]}}} (\mathbb{\hat{E}}r^{S^{(s)}} (\mathpzc{h}^{(s)}) + \mathbb{\hat{E}}r^{S^{(t)}} (\mathpzc{h}^{(t)})) + \epsilon_{1}
\end{gathered}
\end{equation}
Then, combining the inequalities above, with probability greater than $(1 - \delta/2)$: 
\begin{equation}
\begin{gathered}
\mathbb{E}r^{P^{(s)}} ([\mathpzc{h^*}]) \leq \mathbb{E}r^{P^{(s)}} ([\mathpzc{h}^{\#}]) + 2 \epsilon_{1}
\end{gathered}
\end{equation}

Since $\mathpzc{h^\Diamond} \in [\mathpzc{h^*}]$, with probability greater than $(1 - \delta/2)$, $\mathpzc{h^\Diamond}$ will have an error for $P^{(s)}$ which is within $2\epsilon_{2}$ of the best hypothesis there, \textit{i.e.} $\mathbb{E}r^{P_{s}} ([\mathpzc{h^*}])$. Therefore:
\begin{equation}
\begin{gathered}
\mathbb{E}r^{P^{(s)}}(\mathpzc{h^\Diamond}) \leq \inf_{\mathpzc{h}\in \mathbb{H}} \mathbb{E}r^{P^{(s)}}(\mathpzc{h}) + 2(\epsilon_{1} + \epsilon_{2})
\end{gathered}
\end{equation}
\end{proof}

Similar to forward transfer, these bounds depend on the difference between $d_{max}$,  $d_{\mathcal{H}}(n)$, and $VC\mbox{-}dim(\mathbb{H})$. Section \ref{sec:forward} provides details on the meaning of these parameters and their relation to each other.

The main result from Theorem \ref{theom:backward} and its corresponding proof is that an existing source task can also benefit from knowledge acquired during a related target task. This benefit is expected to be smaller than that of transferring forward, since forward transfer benefits from multiple sources (see Eq. $10$) while backward transfer benefits from a single target task (see Eq. $26$). We show that doing backward transfer helps to select a better hypothesis space and therefore provides a better bound on the performance of that task (see Eq. $30$). Therefore, a natural next question is whether backward transfer from a sequence of target tasks, learned one at a time, can help improve these bounds.  we prove that doing backward transfer multiple times sequentially helps to decrease the error on a source task $T^{(s)}$ sequentially as well. 

\begin{definition}
Given classes $\mathcal{F}$ and $\mathcal{H}$, a set of labeled samples $S^{(s)}$ for a source task, and labeled samples $S^{(t)}_{n}$, $S^{(t)}_{n+1}$ for target tasks at times $n$ and $n+1$. Let:
\begin{itemize}
\item $[\mathpzc{h}_{n}^*] \in \mathcal{H}$ be the result of minimising $\inf_{\mathpzc{h}^{(s)}, \mathpzc{h^{(t)}} \in [\mathpzc{h_{n}^*}]} \big( \mathbb{\hat{E}}r^{S^{(s)}} (\mathpzc{h}^{(s)})$ + $\mathbb{\hat{E}}r^{S^{(t)}_{n}} (\mathpzc{h}^{(t)}) \big)$ over all $\mathpzc{[h]} \in \mathcal{H}$, at time $n$.
\item $[\mathpzc{h}_{n+1}^*] \in \mathcal{H}$ be the result of minimising $\inf_{\mathpzc{h}^{(s)}, \mathpzc{h^{(t)}_{n}}, \mathpzc{h^{(t)}_{n+1}} \in [\mathpzc{h_{n}^*}]} \big( \mathbb{\hat{E}}r^{S^{(s)}} (\mathpzc{h}^{(s)})$ + $\mathbb{\hat{E}}r^{S^{(t)}_{n}} (\mathpzc{h}^{(t)}_{n})$ +$\mathbb{\hat{E}}r^{S^{(t)}_{n+1}} (\mathpzc{h}^{(t)}_{n+1}) \big)$ over all $\mathpzc{[h]} \in \mathcal{H}$, at time $n+1$.
\item $\mathpzc{h_{n}^{\Diamond}} \in \mathpzc{[h_{n}^*]} $ that minimises $\mathbb{\hat{E}}r^{S^{(s)}} (\mathpzc{h'})$ over all $\mathpzc{h'} \in \mathpzc{[h_{n}^*]}$, and outputs $\mathpzc{h_{n}^{\Diamond}}$ as the hypothesis for task $T^{(s)}$ at time $n$.
\item $\mathpzc{h_{n+1}^{\Diamond}} \in \mathpzc{[h_{n+1}^*]} $ that minimises $\mathbb{\hat{E}}r^{S^{(s)}} (\mathpzc{h'})$ over all $\mathpzc{h'} \in \mathpzc{[h_{n+1}^*]}$, and outputs $\mathpzc{h_{n+1}^{\Diamond}}$ as the hypothesis for task $T^{(s)}$ at time $n+1$.
\end{itemize}
\label{def:def_backward_multiple}
\end{definition}

\begin{corollary}
Let $P^{(s)}$, $P^{(t)}_{n}$ and $P^{(t)}_{n+1}$ be a set of $\mathcal{F}$-related probability distributions, $S^{(s)}$, $S^{(t)}_{n}$ and $S^{(t)}_{n+1}$ random samples representing these distributions. Let $\mathcal{F}$ and $\mathcal{H}$ be defined as in Definition \ref{def:initial}. Let $d_{max}$ and $d_{\mathcal{H}}(n)$ be defined as in Theorem \ref{theom:forward}. Let $\mathpzc{h_{n}^{\Diamond}}$ and $\mathpzc{h_{n+1}^{\Diamond}}$ be selected according to Definition \ref{def:def_backward_multiple}. Then, for every $\epsilon_{1}$, $\epsilon_{n}$, $\epsilon_{n+1}$, $\delta > 0$, if:
\begin{equation}
\begin{gathered}
|S^{(s)}| \geq \dfrac{64}{\epsilon_{1}^2} \bigg[2d_{max} log\dfrac{12}{\epsilon_{1}} + log\dfrac{8}{\delta}\bigg]
\end{gathered}
\end{equation}
and, at time $n$:
\begin{equation}
\begin{gathered}
|S^{(t)}_{n}| \geq \dfrac{8}{\epsilon_{n}^{2}} \bigg[2d_{\mathcal{H}}(2) log{\dfrac{22}{\epsilon_{n}}} + \dfrac{1}{2} log \dfrac{8}
{\delta}\bigg]
\end{gathered}
\end{equation}
while, at time $n+1$:
\begin{equation}
\begin{gathered}
|S^{(t)}_{n+1}| \geq \dfrac{88}{\epsilon_{n+1}^{2}} \bigg[2d_{\mathcal{H}}(2) log{\dfrac{22}{\epsilon_{n+1}}} + \dfrac{1}{2} log \dfrac{8}
{\delta}\bigg]
\end{gathered}
\end{equation}
then with probability greater than $(1 - \delta)$: 
\begin{equation}
\begin{gathered}
\mathbb{E}r^{P^{(s)}}_{n+1}(\mathpzc{h^\Diamond_{n+1}}) \leq \mathbb{E}r^{P^{(s)}}_{n}(\mathpzc{h^\Diamond_{n}}) + \epsilon_{n} + \epsilon_{n+1} 
\end{gathered}
\end{equation}
\label{cor:backward_order}
\end{corollary}

See Appendix B for the proof of this corollary. These results imply that doing backward transfer sequentially whilst learning target tasks will lead to more refined hypothesis spaces in a source task, beyond the hypothesis space learned initially (with or without forward transfer). Furthermore, this suggests that continually learning $\mathcal{F}$-related tasks while doing both forward and backward transfer can lead to a better bias over the learning environment of these tasks, \textit{i.e.} the result demonstrated by Baxter (2000) for multitask learning, which we demonstrate in the next section.

\section{Continual Learning of Related Tasks using Knowledge Transfer}
\label{sec:continual}
Based on the bounds derived in Section \ref{sec:forward} and Section \ref{sec:backward}, now we are ready to derive bounds of a continual learner that observes supervised related tasks sequentially while doing knowledge transfer. First, lets recall from Definition \ref{def:environment} that the continual learner is embedded in an environment of related tasks, $(\mathcal{P}, \mathcal{Q})$, where $\mathcal{P}$ is the set of all probability distributions on $\mathcal{X} \times Y$ and $\mathcal{Q}$ is a distribution on $\mathcal{P}$. The error of a selected hypothesis space $[\mathpzc{h}^{*}] \in \mathcal{H}$ for all tasks in such environment is defined as:

\begin{equation}
\begin{gathered}
\mathbb{E}r^{\mathcal{Q}}([\mathpzc{h}^{*}]) = \inf_{\mathpzc{h}_{1}, \ldots, \mathpzc{h_{n}} \in [\mathpzc{h^*}]} \sum_{i=1}^{n} \mathbb{E}r^{P_{i}} ([\mathpzc{h_{i}}])
\end{gathered}
\end{equation}

for any $P$ drawn at random from $\mathcal{P}$ according to $\mathcal{Q}$. Let's define $\epsilon_{f}$ as the average $\epsilon$ when performing forward transfer to a new task, \textit{i.e.} $\epsilon_{f}$ corresponds to $2(\epsilon_{1} + \epsilon_{2})$ in Theorem \ref{theom:forward}, averaged across all tasks. Similarly, let's define $\epsilon_{b}$ as the average $\epsilon$ when performing backward transfer to a new task, \textit{i.e.} $\epsilon_{b}$ corresponds to $2(\epsilon_{1} + \epsilon_{2})$ in Theorem \ref{theom:backward} averaged across all tasks. Although the definitions of $\epsilon_{f}$ and $\epsilon_{b}$ oversimplify the continual learner to the case of all tasks achieving roughly the same error bounds by means of transfer, this will serve to demonstrate how forward and backward transfer help to improve the bounds for the continual learner as a whole. For a task $i$, the error bound of applying forward and backward transfer and selecting $[\mathpzc{h}^{*}]$ instead of $[\mathpzc{h}^{\#}]$ as the hypothesis space for that task is:

\begin{equation}
\begin{gathered}
\mathbb{E}r^{P_{i}}([\mathpzc{h}^{*}]) \leq \mathbb{E}r^{P_{i}}([\mathpzc{h}^{\#}]) + (i-1)\epsilon_{f} + (n-i)\epsilon_{b}
\end{gathered}
\label{eq:forward_backward_single}
\end{equation}

As demonstrated in Corollary \ref{cor:forward_order} and \ref{cor:backward_order}, the extent to which transfer helps to improve the error bounds of a particular task $i$ depends on the order of that task in the sequence, which in Eq. \ref{eq:forward_backward_single} impacts the total amount of transfer through $(i-1)$ for forward transfer and $(n-i)$ for backward transfer. Given Eq. \ref{eq:forward_backward_single}, for a sequence of tasks $n$, we can define the error bounds on the environment that learns those tasks by means of transfer as follows. 

\begin{theorem}
Let $\{P_{1}, \ldots, P_{n}\}$ be a set of distributions, one for each task, drawn at random from $\mathcal{P}$, the set of all probability distributions on $\mathcal{X} \times Y$, according to $\mathcal{Q}$, a distribution on $\mathcal{P}$. Let $\mathcal{H}$ be the family of hypothesis spaces for $n$ tasks to be learned in the environment $\mathcal{Q}$, according to Definition \ref{def:initial}, with $[\mathpzc{h}^{*}] \in \mathcal{H}$ selected according to Theorem \ref{theom:forward} and Theorem \ref{theom:backward}. Let $\mathbb{H}$ be the family of hypothesis spaces for $n$ tasks with no transfer, and let $[\mathpzc{h}^{\#}] \in \mathbb{H}$ be selected as the hypothesis space for the $n$ tasks with no transfer. If the number of tasks $n$ satisfies:
\begin{equation}
\begin{gathered}
n \geq \max \Big\{ \dfrac{256}{\epsilon^2} log \dfrac{8\mathcal{C}\Big(\dfrac{\epsilon}{32}, \mathcal{H}^{*}\Big)}{\delta}, \dfrac{64}{\epsilon^2}   \Big\}
\end{gathered}
\end{equation}
with $\mathcal{H^{*}}$ = $\{[\mathpzc{h}^{*}]: \mathpzc{h} \in \mathcal{H}\}$, \textit{i.e.} the set of all hypothesis spaces in the hypothesis space family $\mathcal{H}$ such that each $[\mathpzc{h}^{*}]$ is defined by: 
\begin{equation}
\begin{gathered}
[\mathpzc{h}^{*}](P) = \inf_{\mathpzc{h} \in \mathcal{H}} \mathbb{E}r^{P}(\mathpzc{h})
\end{gathered}
\end{equation}
and, for all $1 \leq i \leq n$, the number of examples per task, $|S_{i}|$ satisfies: 
\begin{equation}
\begin{gathered}
|S_{i}| \geq \max \Big\{ \dfrac{256}{n\epsilon^2} log \dfrac{8\mathcal{C}\Big(\dfrac{\epsilon}{32}, \mathcal{H}_{l}^{n}\Big)}{\delta}, \dfrac{64}{\epsilon^2}   \Big\}
\end{gathered}
\end{equation}

where $\mathcal{H}_{l}^{n}$ = $\cup_{[\mathpzc{h}^{*}] \in \mathcal{H}} [\mathpzc{h}^{*}]_{l}^{n}$ (\textit{i.e.} $\mathcal{H}_{l}^{n}$ is the union of all sequences of hypothesis $[\mathpzc{h}^{*}] \in \mathcal{H}$, each of size $n$, subject to loss function $l$), and with: 
\begin{equation}
\begin{gathered}
\epsilon = \sum_{i=1}^{n} (i-1) \epsilon_{f} + \sum_{i=1}^{n} (n-1) \epsilon_{b}
\end{gathered}
\end{equation}
then, with probability at least $(1 - \delta)$, $[h^{*}] \in \mathcal{H}$ will satisfy:
\begin{equation}
\begin{gathered}
\mathbb{E}r^{\mathcal{Q}}([\mathpzc{h}^{*}]) \leq \mathbb{E}r^{\mathcal{Q}}([\mathpzc{h}^{\#}]) + \epsilon
\end{gathered}
\end{equation}
\label{theom:continual}
\end{theorem}

\begin{proof}
According to Eq. \ref{eq:forward_backward_single}, for all $1 \leq i \leq n$: \begin{equation}
\begin{gathered}
\mathbb{E}r^{P_{i}}([\mathpzc{h}^{*}]) \leq \mathbb{E}r^{P_{i}}([\mathpzc{h}^{\#}]) + (i-1)\epsilon_{f} + (n-i)\epsilon_{b}
\end{gathered}
\end{equation}
which leads to: 
\begin{equation}
\begin{gathered}
\inf_{\mathpzc{h}_{1}, \ldots, \mathpzc{h_{n}} \in [\mathpzc{h^*}]} \sum_{i=1}^{n} \mathbb{E}r^{P_{i}} ([\mathpzc{h_{i}^*}]) \leq \inf_{\mathpzc{h}_{1}, \ldots, \mathpzc{h_{n}} \in [\mathpzc{h}]} \sum_{i=1}^{n} \Big( \mathbb{E}r^{P_{i}} ([\mathpzc{h_{i}^{\#}}])) + (i-1) \epsilon_{f} + (n-1) \epsilon_{b} \Big)
\end{gathered}
\label{eq:bounds_per_task}
\end{equation}
with: 
\begin{equation}
\begin{gathered}
\mathbb{E}r^{\mathcal{Q}}([\mathpzc{h}^{*}]) = \inf_{\mathpzc{h}_{1}, \ldots, \mathpzc{h_{n}} \in [\mathpzc{h^*}]} \sum_{i=1}^{n} \mathbb{E}r^{P_{i}} ([\mathpzc{h^{*}_{i}}])
\end{gathered}
\end{equation}
and:
\begin{equation}
\begin{gathered}
\mathbb{E}r^{\mathcal{Q}}([\mathpzc{h}^{\#}]) = \inf_{\mathpzc{h}_{1}, \ldots, \mathpzc{h_{n}} \in [\mathpzc{h}]} \sum_{i=1}^{n} \mathbb{E}r^{P_{i}} ([\mathpzc{h}_{i}^{\#}])
\end{gathered}
\end{equation}
then:
\begin{equation}
\begin{gathered}
\mathbb{E}r^{\mathcal{Q}}([\mathpzc{h}^{*}]) \leq \mathbb{E}r^{\mathcal{Q}}([\mathpzc{h}^{\#}]) + \epsilon
\end{gathered}
\end{equation}
\end{proof}

Theorem \ref{theom:continual} and its corresponding proof provide the most relevant result of our framework: learning a set of tasks continually with forward and backward transfer will lead to incrementally learning a better bias over the environment itself. Furthermore, the larger the number of tasks $n$, the better the bounds. This can be inferred from Eq. \ref{eq:bounds_per_task}, since for a particular $\mathbb{E}r^{P_{i}} ([\mathpzc{h_{i}^{\#}}]))$ on the right-hand side the larger the number of tasks the larger $i$ and $n$ are, the larger the difference, or gap, with the left-hand side. This implies that the hypothesis space for a particular task which is selected by considering other tasks, \textit{i.e.} via transfer, is a better hypothesis space than would be selected by learning that task in isolation. Since this occurs for all tasks in the environment, the better the bias learned over that environment will be and therefore the better future tasks will be learned, leading to an increasingly knowledgeable system.

As a final remark, note that in practice the bounds in Theorem \ref{theom:continual} depend on the samples $S_{i}$, $1 \leq i \leq n$, drawn from the corresponding $P_{i}$ probability distributions, since this is the data that can be observed while learning. For these bounds to apply, for all $1 \leq i \leq n$, the bounds between $P_{i}$ and $S_{i}$ must satisfy \citep{baxter2000model}:

\begin{equation}
\begin{gathered}
\mathbb{E}r^{P_{i}}(h) \leq \mathbb{\hat{E}}r^{S_{i}}(h) + \Big[ \dfrac{32}{m} \Big( dlog\dfrac{2\epsilon m}{d} + log \dfrac{2}{\delta} \Big) \Big]^{1/2}
\end{gathered}
\label{eq:err_ep_es}
\end{equation}

with $d$ the VC-dimension of $\mathcal{H}$ and $m$ the number of examples. Then, with probability at least $(1 - \delta)$ all $h \in \mathcal{H}$ will satisfy Eq. \ref{eq:err_ep_es}.

\section{Experiments}
\label{sec:experiments}
We experiment with an example inspired by multitask learning \citep{zhang2014regularization}. A continual learner observes a set of regression tasks to learn four functions, three of which are linear and related while one is unrelated (see Figure \ref{fig:diagram2}). We report six different scenarios to prove bounds presented in previous sections: 1) forward transfer from $f_{1}$ to $f_{2}$ (bound in Theorem \ref{theom:forward}), 2) backward transfer from $f_{2}$ to $f_{1}$ (bound in Theorem \ref{theom:backward}), 3) forward transfer from $f_{1}$ and $f_{2}$ to $f_{3}$ (bound in Corollary \ref{cor:forward_order}), 4) backward transfer from $f_{2}$ to $f_{1}$ and then from $f_{3}$ to $f_{1}$ (bound in Corollary \ref{cor:backward_order}), 5) forward transfer from $f_{1}$ to $f_{4}$ (bound in Theorem \ref{theom:forward} for an unrelated task), and 6) backward transfer from $f_{4}$ to $f_{1}$ (bound in Theorem \ref{theom:backward} for an unrelated task). Each of these six scenarios is trained and tested independently from the other scenarios.  We measure the final $R^{2}$ of the three tasks learned by that continual learner. We use a neural network with $1$ hidden layer of $10$ units to learn in each of these scenarios. For each scenario, the task from which transfer occurs is trained only partially (\textit{i.e.} before full convergence), while the task which is subject to transfer is trained until convergence. We generate $30$ random examples for each task, for values of $x$ between $0$ and $10$. We add Gaussian noise with mean $1$ and standard deviation $2$. We split data from each task into training ($75\%$) and test ($25\%$) sets. We repeat sampling, splitting, training and testing $10$ times. Table \ref{tab:experiments} shows that transfer across related tasks (scenarios $1$ to $4$) benefits $R^{2}$ performance. Our main finding is that, by training source tasks only partially, we are able to keep the hypothesis space large enough for the backward/forward transfer to have an effect. This also allows exploring a larger set of $\mathcal{F}$ transformation functions between hypothesis spaces, which appears to be critical for transfer. The practical implication of this is that we will need to store partially converged versions of each task's model for future transfer.

\begin{figure}
\centering
\includegraphics[scale = 0.32]{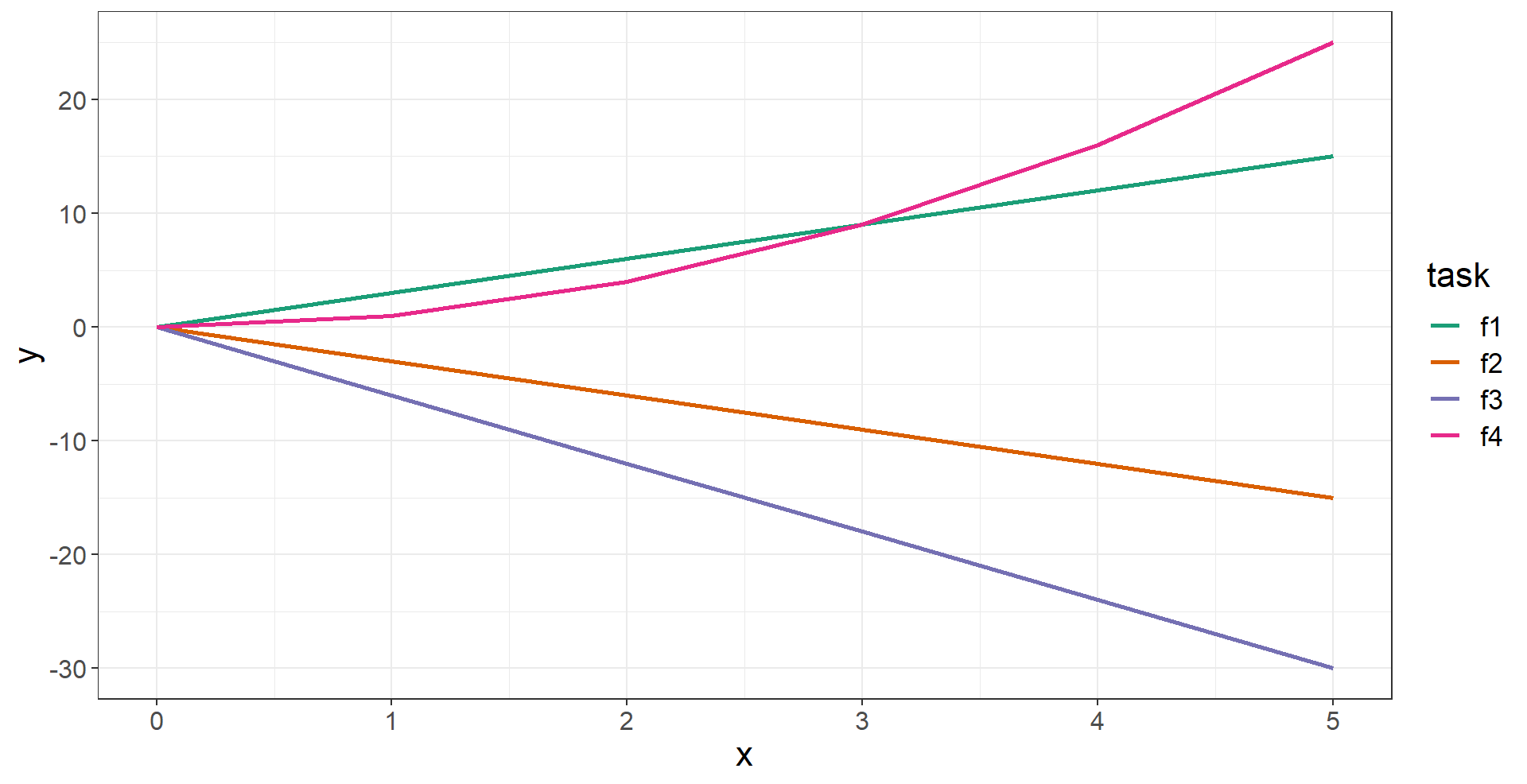} 
\caption{A set of problems or tasks. Three of these are linear functions and are related, while one is an unrelated task.}
\label{fig:diagram2}
\end{figure}

\begin{table}
\centering
\small
\begin{tabular}{|c|c|c|c|c|c|}
     \hline \textbf{Scenario} & $R^{2} f_{1}$ & $R^{2} f_{2}$ & $R^{2} f_{3}$ & $R^{2} f_{4}$ & $R^{2} f_{1}, f_{2}, f_{3}$ \\ \hline
     Isolated learning & $0.9997\pm0.0002$ & $0.999\pm0.001$ & $0.999\pm0.001$ & $0.9842\pm0.0382$ & $0.9955\pm0.0101$ \\ \hline
     Isolated learning (noise) & $0.8463\pm0.033$ & $0.8737\pm0.0127$ & $0.896\pm0.0133$ & $0.7825\pm0.0754$ & $0.8496\pm0.0034$ \\ \hline
     Scenario 1 (noise) & $--$ & $0.8919\pm0.0174$ & $--$ & $--$ & $--$ \\ \hline
     Scenario 2 (noise) & $0.9613\pm0.0181$ & $--$ & $--$ & $--$  & $--$ \\ \hline
     Scenario 3 (noise) & $--$ & $--$ & $0.9137\pm0.0135$ & $--$  & $--$ \\ \hline
     Scenario 4 (noise) & $0.9636\pm0.0172$ & $--$ & $--$ & $--$ & $--$ \\ \hline
     Scenario 5 (noise) & $--$ & $--$ & $--$ & $0.4322\pm0.0933$  & $--$ \\ \hline
     Scenario 6 (noise) & $0.7035\pm0.0364$ & $--$ & $--$ & $--$ & $--$ \\ \hline
\end{tabular}
\caption{Mean $R^{2}$, and their standard deviations, of six transfer scenarios on a toy example of four functions: $y_{1} = -3x + 10$, $y_{2} = -3x - 5$, $y_{3} = -6x - 12$ and $y_{4} = x^{2}$ ($--$ denotes not applicable).}
\label{tab:experiments}
\end{table}
\vspace{-2mm}
\section{Discussion and Conclusion}
\label{sec:discussion_conclusion}
We proposed a theory for knowledge transfer in supervised continual learning. We aim to encourage further research in knowledge transfer for achieving increasingly knowledgeable continual learning systems. Our proposed framework relies on the assumption of relatedness among tasks in a specific environment. This assumption may be applicable to a variety of domains that learn different but related tasks, \textit{e.g.} tasks in a clinical domain, tasks in manufacturing, etc.  Our error bounds are agnostic to the implementation of the continual learner. One would naturally wonder how these bounds apply to implementations with deep neural networks. Since those bounds depend on the number of examples per task, which itself depends on the number of tasks (see Eq. \ref{eq:number_examples_tasks}), it is possible that some data from previous tasks will be needed. Strategies such as memories per task or generative models, which have been used in several studies, could be helpful. We also believe that modular or semi-modular networks, with specialized components for each task, will potentially be necessary for effective knowledge transfer. Furthermore, having specialized modules for representing the set of transformations between hypothesis spaces for each task could be helpful. We also hypothesise that scenarios such as class-incremental learning of related classes could benefit more from transfer than scenarios such as task or domain-incremental learning, although more research is required in this direction. 

Recent work suggests that the framework of the VC-dimension is not appropriate for deep neural networks. Other frameworks based on infinite-width networks \citep{golikov2020towards} and robustness-based networks \citep{bubeck2021law} have been widely studied for deep neural networks that observe examples of all tasks at once. Although most work in supervised continual learning has used deep neural networks, the infinite-width and robustness-based frameworks have not been analysed with the lens of learning incrementally. Intuitively, overparameterised networks would imply a bigger challenge for continual learning, as this overparameterisation would lead to excellent performance on a particular task, making the network harder to adapt to subsequent ones. Extending studies on infinite-width or robustness-based networks to the challenges of continual learning would be an interesting avenue of research.

\bibliography{collas2022_conference}

\begin{thebibliography}{34}
\providecommand{\natexlab}[1]{#1}
\providecommand{\url}[1]{\texttt{#1}}
\expandafter\ifx\csname urlstyle\endcsname\relax
  \providecommand{\doi}[1]{doi: #1}\else
  \providecommand{\doi}{doi: \begingroup \urlstyle{rm}\Url}\fi

\bibitem[Baxter et~al.(2000)]{baxter2000model}
Jonathan Baxter et~al.
\newblock A model of inductive bias learning.
\newblock \emph{Journal of Artificial Intelligence Research}, 12\penalty0
  (149-198):\penalty0 3, 2000.

\bibitem[Ben-David \& Borbely(2008)Ben-David and Borbely]{ben2008notion}
Shai Ben-David and Reba~Schuller Borbely.
\newblock A notion of task relatedness yielding provable multiple-task learning
  guarantees.
\newblock \emph{Machine learning}, 73\penalty0 (3):\penalty0 273--287, 2008.

\bibitem[Benavides-Prado(2020)]{benavides2020beyond}
Diana Benavides-Prado.
\newblock Beyond catastrophic forgetting in continual learning: An attempt with
  svm.
\newblock In \emph{ICML}. The International Conference on Machine Learning
  (ICML), 2020.

\bibitem[Benavides-Prado et~al.(2017)Benavides-Prado, Koh, and
  Riddle]{benavides2017accgensvm}
Diana Benavides-Prado, Yun~Sing Koh, and Patricia Riddle.
\newblock Accgensvm: Selectively transferring from previous hypotheses.
\newblock In \emph{Proc. Intern. Joint Conf. Artificial Intel}, pp.\
  1440--1446, 2017.

\bibitem[Benavides-Prado et~al.(2020)Benavides-Prado, Koh, and
  Riddle]{benavides2020towards}
Diana Benavides-Prado, Yun~Sing Koh, and Patricia Riddle.
\newblock Towards knowledgeable supervised lifelong learning systems.
\newblock \emph{Journal of Artificial Intelligence Research}, 68:\penalty0
  159--224, 2020.

\bibitem[Bubeck et~al.(2021)Bubeck, Li, and Nagaraj]{bubeck2021law}
S{\'e}bastien Bubeck, Yuanzhi Li, and Dheeraj~M Nagaraj.
\newblock A law of robustness for two-layers neural networks.
\newblock In \emph{Conference on Learning Theory}, pp.\  804--820. PMLR, 2021.

\bibitem[Chen \& Liu(2018)Chen and Liu]{chen2018lifelong}
Zhiyuan Chen and Bing Liu.
\newblock Lifelong machine learning.
\newblock \emph{Synthesis Lectures on Artificial Intelligence and Machine
  Learning}, 12\penalty0 (3):\penalty0 1--207, 2018.

\bibitem[Delange et~al.(2021)Delange, Aljundi, Masana, Parisot, Jia, Leonardis,
  Slabaugh, and Tuytelaars]{delange2021continual}
Matthias Delange, Rahaf Aljundi, Marc Masana, Sarah Parisot, Xu~Jia, Ales
  Leonardis, Greg Slabaugh, and Tinne Tuytelaars.
\newblock A continual learning survey: Defying forgetting in classification
  tasks.
\newblock \emph{IEEE Transactions on Pattern Analysis and Machine
  Intelligence}, 2021.

\bibitem[Golikov(2020)]{golikov2020towards}
Eugene Golikov.
\newblock Towards a general theory of infinite-width limits of neural
  classifiers.
\newblock In \emph{International Conference on Machine Learning}, pp.\
  3617--3626. PMLR, 2020.

\bibitem[Hung et~al.(2019)Hung, Tu, Wu, Chen, Chan, and
  Chen]{hung2019compacting}
Ching-Yi Hung, Cheng-Hao Tu, Cheng-En Wu, Chien-Hung Chen, Yi-Ming Chan, and
  Chu-Song Chen.
\newblock Compacting, picking and growing for unforgetting continual learning.
\newblock \emph{Advances in Neural Information Processing Systems}, 32, 2019.

\bibitem[Ke et~al.(2020)Ke, Liu, and Huang]{ke2020continual}
Zixuan Ke, Bing Liu, and Xingchang Huang.
\newblock Continual learning of a mixed sequence of similar and dissimilar
  tasks.
\newblock \emph{Advances in Neural Information Processing Systems},
  33:\penalty0 18493--18504, 2020.

\bibitem[Ke et~al.(2021)Ke, Liu, Ma, Xu, and Shu]{ke2021achieving}
Zixuan Ke, Bing Liu, Nianzu Ma, Hu~Xu, and Lei Shu.
\newblock Achieving forgetting prevention and knowledge transfer in continual
  learning.
\newblock \emph{Advances in Neural Information Processing Systems}, 34, 2021.

\bibitem[Kirkpatrick et~al.(2017)Kirkpatrick, Pascanu, Rabinowitz, Veness,
  Desjardins, Rusu, Milan, Quan, Ramalho, Grabska-Barwinska,
  et~al.]{kirkpatrick2017overcoming}
James Kirkpatrick, Razvan Pascanu, Neil Rabinowitz, Joel Veness, Guillaume
  Desjardins, Andrei~A Rusu, Kieran Milan, John Quan, Tiago Ramalho, Agnieszka
  Grabska-Barwinska, et~al.
\newblock Overcoming catastrophic forgetting in neural networks.
\newblock \emph{Proceedings of the national academy of sciences}, pp.\
  201611835, 2017.

\bibitem[Knoblauch et~al.(2020)Knoblauch, Husain, and
  Diethe]{knoblauch2020optimal}
Jeremias Knoblauch, Hisham Husain, and Tom Diethe.
\newblock Optimal continual learning has perfect memory and is np-hard.
\newblock In \emph{International Conference on Machine Learning}, pp.\
  5327--5337. PMLR, 2020.

\bibitem[Lampinen \& Ganguli(2018)Lampinen and Ganguli]{lampinen2018analytic}
Andrew~K Lampinen and Surya Ganguli.
\newblock An analytic theory of generalization dynamics and transfer learning
  in deep linear networks.
\newblock \emph{arXiv preprint arXiv:1809.10374}, 2018.

\bibitem[Lee \& Lee(2020)Lee and Lee]{lee2020clinical}
Cecilia~S Lee and Aaron~Y Lee.
\newblock Clinical applications of continual learning machine learning.
\newblock \emph{The Lancet Digital Health}, 2\penalty0 (6):\penalty0
  e279--e281, 2020.

\bibitem[Lee et~al.(2021)Lee, Goldt, and Saxe]{lee2021continual}
Sebastian Lee, Sebastian Goldt, and Andrew Saxe.
\newblock Continual learning in the teacher-student setup: Impact of task
  similarity.
\newblock In \emph{International Conference on Machine Learning}, pp.\
  6109--6119. PMLR, 2021.

\bibitem[Lesort et~al.(2020)Lesort, Lomonaco, Stoian, Maltoni, Filliat, and
  D{\'\i}az-Rodr{\'\i}guez]{lesort2020continual}
Timoth{\'e}e Lesort, Vincenzo Lomonaco, Andrei Stoian, Davide Maltoni, David
  Filliat, and Natalia D{\'\i}az-Rodr{\'\i}guez.
\newblock Continual learning for robotics: Definition, framework, learning
  strategies, opportunities and challenges.
\newblock \emph{Information fusion}, 58:\penalty0 52--68, 2020.

\bibitem[Lopez-Paz \& Ranzato(2017)Lopez-Paz and Ranzato]{lopez2017gradient}
David Lopez-Paz and Marc'Aurelio Ranzato.
\newblock Gradient episodic memory for continual learning.
\newblock In \emph{Advances in Neural Information Processing Systems}, pp.\
  6467--6476, 2017.

\bibitem[Maschler et~al.(2021)Maschler, Pham, and
  Weyrich]{maschler2021regularization}
Benjamin Maschler, Thi Thu~Huong Pham, and Michael Weyrich.
\newblock Regularization-based continual learning for anomaly detection in
  discrete manufacturing.
\newblock \emph{Procedia CIRP}, 104:\penalty0 452--457, 2021.

\bibitem[New et~al.(2022)New, Baker, Nguyen, and Vallabha]{new2022lifelong}
Alexander New, Megan Baker, Eric Nguyen, and Gautam Vallabha.
\newblock Lifelong learning metrics.
\newblock \emph{arXiv preprint arXiv:2201.08278}, 2022.

\bibitem[Ostapenko et~al.(2019)Ostapenko, Puscas, Klein, Jahnichen, and
  Nabi]{ostapenko2019learning}
Oleksiy Ostapenko, Mihai Puscas, Tassilo Klein, Patrick Jahnichen, and Moin
  Nabi.
\newblock Learning to remember: A synaptic plasticity driven framework for
  continual learning.
\newblock In \emph{Proceedings of the IEEE/CVF Conference on Computer Vision
  and Pattern Recognition}, pp.\  11321--11329, 2019.

\bibitem[Raczkowski \& Sadowski(1990)Raczkowski and
  Sadowski]{raczkowski1990equivalence}
Konrad Raczkowski and Pawe{\l} Sadowski.
\newblock Equivalence relations and classes of abstraction.
\newblock \emph{Formalized Mathematics}, 1\penalty0 (3):\penalty0 441--444,
  1990.

\bibitem[Riemer et~al.(2018)Riemer, Cases, Ajemian, Liu, Rish, Tu, and
  Tesauro]{riemer2018learning}
Matthew Riemer, Ignacio Cases, Robert Ajemian, Miao Liu, Irina Rish, Yuhai Tu,
  and Gerald Tesauro.
\newblock Learning to learn without forgetting by maximizing transfer and
  minimizing interference.
\newblock \emph{arXiv preprint arXiv:1810.11910}, 2018.

\bibitem[Ring(1997)]{ring1997child}
Mark~B Ring.
\newblock Child: A first step towards continual learning.
\newblock \emph{Machine Learning}, 28\penalty0 (1):\penalty0 77--104, 1997.

\bibitem[Rish(2022)]{irina2022}
Irina Rish.
\newblock Towards general and robust ai at scale, 2022.
\newblock URL \url{https://www.youtube.com/watch?v=OL3hUZh6lTc}.

\bibitem[Rostami et~al.(2020)Rostami, Isele, and Eaton]{rostami2020using}
Mohammad Rostami, David Isele, and Eric Eaton.
\newblock Using task descriptions in lifelong machine learning for improved
  performance and zero-shot transfer.
\newblock \emph{Journal of Artificial Intelligence Research}, 67:\penalty0
  673--704, 2020.

\bibitem[Van~de Ven \& Tolias(2019)Van~de Ven and Tolias]{van2019three}
Gido~M Van~de Ven and Andreas~S Tolias.
\newblock Three scenarios for continual learning.
\newblock \emph{arXiv preprint arXiv:1904.07734}, 2019.

\bibitem[van~de Ven et~al.(2020)van~de Ven, Siegelmann, and
  Tolias]{van2020brain}
Gido~M van~de Ven, Hava~T Siegelmann, and Andreas~S Tolias.
\newblock Brain-inspired replay for continual learning with artificial neural
  networks.
\newblock \emph{Nature communications}, 11\penalty0 (1):\penalty0 1--14, 2020.

\bibitem[Vogelstein et~al.(2020)Vogelstein, Dey, Helm, LeVine, Mehta, Geisa,
  Xu, van~de Ven, Chang, Gao, et~al.]{vogelstein2020representation}
Joshua~T Vogelstein, Jayanta Dey, Hayden~S Helm, Will LeVine, Ronak~D Mehta,
  Ali Geisa, Haoyin Xu, Gido~M van~de Ven, Emily Chang, Chenyu Gao, et~al.
\newblock Representation ensembling for synergistic lifelong learning with
  quasilinear complexity.
\newblock \emph{arXiv preprint arXiv:2004.12908}, 2020.

\bibitem[Weinshall et~al.(2018)Weinshall, Cohen, and
  Amir]{weinshall2018curriculum}
Daphna Weinshall, Gad Cohen, and Dan Amir.
\newblock Curriculum learning by transfer learning: Theory and experiments with
  deep networks.
\newblock In \emph{International Conference on Machine Learning}, pp.\
  5238--5246. PMLR, 2018.

\bibitem[Yoon et~al.(2017)Yoon, Yang, et~al.]{yoon2017lifelong}
Jaehong Yoon, Eunho Yang, et~al.
\newblock Lifelong {L}earning with {D}ynamically {E}xpandable {N}etworks.
\newblock \emph{arXiv:1708.01547}, 2017.

\bibitem[Zeng et~al.(2019)Zeng, Chen, Cui, and Yu]{zeng2019continual}
Guanxiong Zeng, Yang Chen, Bo~Cui, and Shan Yu.
\newblock Continual learning of context-dependent processing in neural
  networks.
\newblock \emph{Nature Machine Intelligence}, 1\penalty0 (8):\penalty0
  364--372, 2019.

\bibitem[Zhang \& Yeung(2014)Zhang and Yeung]{zhang2014regularization}
Yu~Zhang and Dit-Yan Yeung.
\newblock A regularization approach to learning task relationships in multitask
  learning.
\newblock \emph{ACM Transactions on Knowledge Discovery from Data (TKDD)},
  8\penalty0 (3):\penalty0 1--31, 2014.

\end{thebibliography}
\bibliographystyle{collas2022_conference}

\newpage
\section*{Appendix A}
\label{appendix:A}
Proof for Corollary \ref{cor:forward_order}:
\begin{proof} 
By Theorem 2 in Ben-David and Borbely, at time $n+z$:
\begin{equation}
\begin{gathered}
\mathbb{E}r^{P^{(t)}}_{n+z}([\mathpzc{h}^{*}_{n+z}]) \leq \inf_{\mathpzc{h}_{1}, \ldots, \mathpzc{h_{n+z}} \in [\mathpzc{h_{n+z}^*}]} \dfrac{1}{n+z} \sum_{i=1}^{n+z}\mathbb{\hat{E}}r^{S^{(s)}_{i}} (\mathpzc{h_{i}}) + \epsilon_{n+z}
\end{gathered}
\end{equation}
while, at time $n$:
\begin{equation}
\begin{gathered}
\mathbb{E}r^{P^{(t)}}_{n}([\mathpzc{h}^{*}_{n}]) \leq \inf_{\mathpzc{h}_{1}, \ldots, \mathpzc{h_{n}} \in [\mathpzc{h_{n}^*}]} \dfrac{1}{n} \sum_{i=1}^{n}\mathbb{\hat{E}}r^{S^{(s)}_{i}} (\mathpzc{h_{i}}) + \epsilon_{n}
\end{gathered}
\end{equation}
And by Ben-David and Borbely (2008), and also Baxter (2000):
\begin{equation}
\begin{gathered}
\inf_{\mathpzc{h}_{1}, \ldots, \mathpzc{h_{n+z}} \in [\mathpzc{h_{n+z}^*}]} \dfrac{1}{n+z} \sum_{i=1}^{n+z}\mathbb{\hat{E}}r^{S^{(s)}_{i}} (\mathpzc{h}_{i})  \leq \inf_{\mathpzc{h}_{1}, \ldots, \mathpzc{h_{n}} \in [\mathpzc{h_{n}^*}]} \dfrac{1}{n} \sum_{i=1}^{n}\mathbb{\hat{E}}r^{S^{(s)}_{i}} (\mathpzc{h_{i}})  
\end{gathered}
\label{eq:inf_nz_n}
\end{equation}
Therefore:
\begin{equation}
\begin{gathered}
\mathbb{E}r^{P^{(t)}}_{n+z}([\mathpzc{h}^{*}_{n+z}]) \leq \mathbb{E}r^{P^{(t)}}_{n}([\mathpzc{h}^{*}_{n}]) + \epsilon_{n} + \epsilon_{n+z}
\end{gathered}
\end{equation}
With $\mathbb{E}r^{P^{(t)}}_{n+z}(\mathpzc{h}^{\diamond}_{n+z})$ the best hypothesis in $\mathbb{E}r^{P^{(t)}}_{n+z}([\mathpzc{h}^{*}_{n+z}])$ and $\mathbb{E}r^{P^{(t)}}_{n}(\mathpzc{h}^{\diamond}_{n})$ the best hypothesis in $\mathbb{E}r^{P^{(t)}}_{n}([\mathpzc{h}^{*}_{n}])$, the theorem is proved.
\end{proof}

\newpage
\section*{Appendix B}
\label{appendix:b}
Proof for Corollary \ref{cor:backward_order}:
\begin{proof} 
By Theorem 2 in Ben-David and Borbely, at time $n+1$:
\begin{equation}
\begin{gathered}
\mathbb{E}r^{P^{(s)}}_{n+1}([\mathpzc{h}^{*}_{n+1}]) \leq \inf_{\mathpzc{h}^{(s)}, \mathpzc{h}^{(t)}_{n}, \mathpzc{h}^{(t)}_{n+1} \in [\mathpzc{h_{n+1}^*}]} \big( \mathbb{\hat{E}}r^{S^{(s)}} (\mathpzc{h}^{(s)}) + \mathbb{\hat{E}}r^{S^{(t)}_{n}} (\mathpzc{h}^{(t)}_{n}) + \mathbb{\hat{E}}r^{S^{(t)}_{n+1}} (\mathpzc{h}^{(t)}_{n+1}) \big) +\epsilon_{n+1}
\end{gathered}
\end{equation}
while, at time $n$:
\begin{equation}
\begin{gathered}
\mathbb{E}r^{P^{(s)}}_{n}([\mathpzc{h}^{*}_{n}]) \leq \inf_{\mathpzc{h}^{(s)}, \mathpzc{h}^{(t)}_{n} \in [\mathpzc{h_{n}^*}]} \big( \mathbb{\hat{E}}r^{S^{(s)}} (\mathpzc{h}^{(s)}) + \mathbb{\hat{E}}r^{S^{(t)}_{n}} (\mathpzc{h}^{(t)}_{n}) \big) + \epsilon_{n}\end{gathered}
\end{equation}
And by Ben-David and Borbely (2008), and also Baxter (2000):
\begin{equation}
\begin{gathered}
\inf_{\mathpzc{h}_{1}, \ldots, \mathpzc{h_{n+1}} \in [\mathpzc{h_{n+1}^*}]} \dfrac{1}{n+1} \sum_{i=1}^{n+1}\mathbb{\hat{E}}r^{S^{(s)}_{i}} (\mathpzc{h}_{i}) \leq \inf_{\mathpzc{h}_{1}, \ldots, \mathpzc{h_{n}} \in [\mathpzc{h_{n}^*}]} \dfrac{1}{n} \sum_{i=1}^{n}\mathbb{\hat{E}}r^{S^{(s)}_{i}} (\mathpzc{h}_{i}) 
\end{gathered}
\end{equation}
Therefore:
\begin{equation}
\begin{gathered}
\mathbb{E}r^{P^{(s)}}_{n+1}([\mathpzc{h}^{*}_{n+1}]) \leq \mathbb{E}r^{P^{(s)}}_{n}([\mathpzc{h}^{*}_{n}]) + \epsilon_{n} + \epsilon_{n+1}
\end{gathered}
\end{equation}
With $\mathbb{E}r^{P^{(s)}}_{n+1}(\mathpzc{h}^{\diamond}_{n+1})$ the best hypothesis in $\mathbb{E}r^{P^{(s)}}_{n+1}([\mathpzc{h}^{*}_{n+1}])$ and $\mathbb{E}r^{P^{(s)}}_{n}(\mathpzc{h}^{\diamond}_{n})$ the best hypothesis in $\mathbb{E}r^{P^{(s)}}_{n}([\mathpzc{h}^{*}_{n}])$, the theorem is proved.
\end{proof}

\end{document}